\documentclass{article}
\usepackage{iclr2027_conference}
\usepackage{times}
\usepackage{amsmath,amssymb,amsthm}
\usepackage{graphicx}
\usepackage{booktabs}
\usepackage{xcolor}
\usepackage{enumitem}
\usepackage{hyperref}
\newtheorem{theorem}{Theorem}[section]
\newtheorem{lemma}[theorem]{Lemma}

\newtheorem{corollary}[theorem]{Corollary}

\newtheorem{remark}[theorem]{Remark}

\newcommand{\R}{\mathbb{R}}
\newcommand{\E}{\mathbb{E}}
\renewcommand{\Pr}{\mathbb{P}}
\newcommand{\norm}[1]{\left\|#1\right\|}
\newcommand{\inner}[2]{\langle #1, #2 \rangle}

\title{How Should Reasoning Be Organized in a Transformer's Latent Space?}

\author{%
\parbox[t]{0.30\textwidth}{\centering\normalfont
\textbf{Hongyu Gu}\textsuperscript{1}\\
University of Science\\
and Technology of China\\
Hefei, China\\[0.35em]
{\scriptsize\texttt{ustc\_23ghy@mail.ustc.edu.cn}}}%
\And
\parbox[t]{0.30\textwidth}{\centering\normalfont
\textbf{Chang Liu}\textsuperscript{2}\\
Zhongguancun Academy\\
\phantom{and Technology of China}\\
Beijing, China\\[0.35em]
{\scriptsize\texttt{liuchang@bza.edu.cn}}}%
\And
\parbox[t]{0.30\textwidth}{\centering\normalfont
\textbf{Jingwen Fu}\textsuperscript{2}\thanks{Corresponding author.}\\
Zhongguancun Academy\\
\phantom{and Technology of China}\\
Beijing, China\\[0.35em]
{\scriptsize\texttt{jwfu99@gmail.com}}}%
}

\iclrfinalcopy

\begin{document}
\maketitle
\fancyhead{}
\lhead{Preprint}
%
\begin{abstract} 
Continuous reasoning has emerged as a promising way to improve reasoning in large language models (LLMs). Yet we still lack a clear principle for deciding what a latent state should preserve. Reasoning by superposition shows that a single latent state can encode several search alternatives and expand them in parallel. We ask how those states should be weighted as reasoning proceeds. A natural choice is to preserve only the states active at the frontier step, since keeping every reached state appears to spread a limited hidden width too thin. We show that the opposite can hold. When later computation draws on several reached states, a cumulative state can guide attention correctly at a smaller hidden width than a frontier state that stores fewer states. At the same width, the cumulative state therefore keeps more intermediate states available for later reasoning. More generally, equal cumulative weights are optimal when future queries are unknown and remain close to the best task-specific weights when those queries are known. Experiments with two-layer and GPT-2 Transformers reproduce the predicted width advantage and show that unequal weights fail first on the states that receive the least weight. This suggests a important principle: keep reached states equally weighted, and restore equal weights as computation proceeds. 

\end{abstract}

\section{Introduction} \label{sec:intro} 
What should a latent state preserve as reasoning unfolds? This question matters as language models move part of their reasoning from written tokens into continuous computation. Continuous chain-of-thought feeds latent states back into the model~\citep{hao2024coconut}; recurrent Transformers update a latent state over several steps~\citep{xu2024looped,fu2026fullylooped}; and reasoning by superposition shows that one latent state can represent several search alternatives at once~\citep{zhu2025reasoning}. These results establish that superposed reasoning is possible, but they leave open what the latent state should preserve as computation continues. 

We compare three choices. A \emph{frontier state} keeps only the states active at the current step. A \emph{cumulative state} keeps all states reached so far with equal weight. A general \emph{weighted state} keeps the same set of reached states but assigns them different nonnegative weights. All three use the same task, candidate answers, Transformer, and final query.

The usual capacity intuition favors the frontier state. When more states share one normalized latent state, each receives less weight and should become harder to recover~\citep{elhage2022toy}. Reasoning, however, rarely uses one earlier state in isolation. A later step may combine information from several reached states. Their useful components reinforce the correct answer, while unrelated candidates remain random interference. This leads to our central result: although it stores more states, the cumulative state can work at a smaller hidden width than the frontier state. 

\paragraph{Cumulative-state hypothesis.} When any reached state may be needed again, the latent state should assign equal weight to all reached states. Let $\mathfrak D(\boldsymbol a;\mathcal U)$ denote the smallest hidden width that meets the target success probability throughout a $D$-step computation for every query in $\mathcal U$. If a query assigns at most $C/\sqrt m$ weight to any one of the $m$ reached states, let $\mathcal U_C$ contain all queries that satisfy this bound. We prove \begin{equation} \mathfrak D(\boldsymbol a_{\mathrm{cum}};\mathcal U) \le C^2\inf_{\boldsymbol a}\mathfrak D(\boldsymbol a;\mathcal U), \qquad \boldsymbol a_{\mathrm{cum}} \in \arg\min_{\boldsymbol a}\mathfrak D(\boldsymbol a;\mathcal U_C). \end{equation} 
Thus, for a known query family, equal cumulative weights require at most $C^2$ times the hidden width of the best task-specific weights. When only the coefficient bound is known, equal cumulative weights are the best fixed choice in the worst case over the full class $\mathcal U_C$. 

\begin{figure}[t] \centering \includegraphics[width=\textwidth]{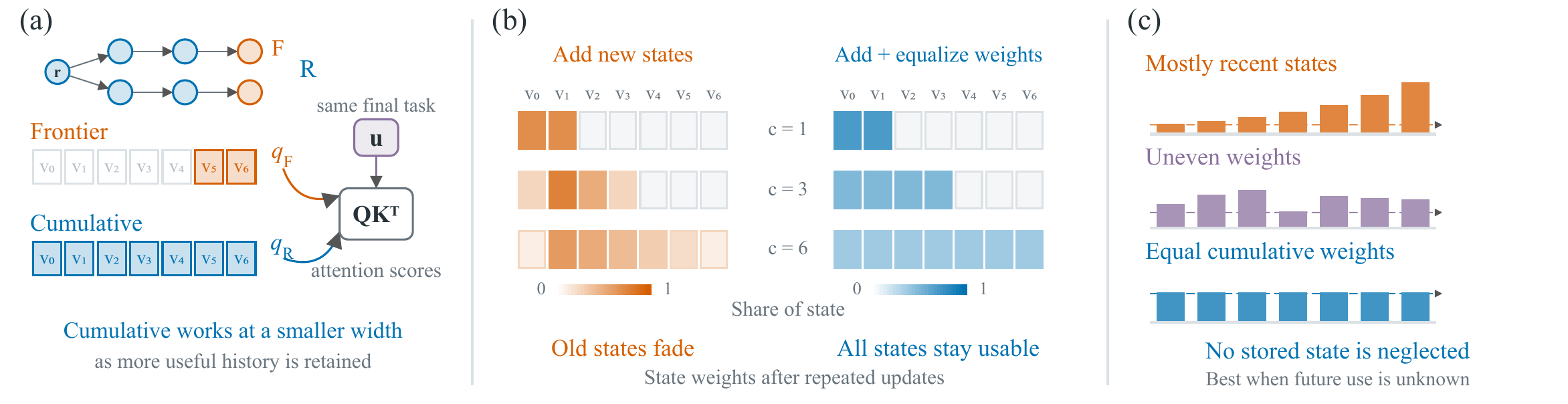} \caption{\textbf{The cumulative-state hypothesis from latent state to computation.} (a) The frontier state keeps only the states active now, whereas the cumulative state keeps all reached states for the same final task. (b) Each update adds new states; equalizing the weights keeps every reached state available. (c) Equal cumulative weights avoid a weakly represented state and are optimal when future use is unknown.} 
\label{fig:theory-overview} 
\end{figure}  

The query condition covers computations that combine information across a set, add costs along a route, execute an algorithm, or evaluate intermediate steps~\citep{zaheer2017deepsets,lee2019settransformer,kool2019routing,velickovic2022clrs,lightman2024verify}. We test the predictions with synthetic query families, paired width sweeps in a lightweight Transformer and a model built from GPT-2 blocks, and position-level interventions. Across these tests, the cumulative state reaches the target accuracy at smaller width, while unequal weights fail first at the states assigned the least weight. 

\paragraph{Contributions.} Our main contributions are: \begin{enumerate}[leftmargin=*,itemsep=2pt] \item \textbf{A counterintuitive result about latent-state capacity.} We show when preserving more reached states makes reasoning easier: their useful components reinforce later attention, allowing the cumulative state to work at a smaller hidden width than the frontier state. \item \textbf{A Transformer account of the full reasoning process.} We track how attention retrieves states, values carry new information, the residual preserves earlier states, and the output projection sets their weights across all $D$ steps. This analysis identifies the least represented state as the bottleneck for later computation. \item \textbf{An optimal rule for weighting reached states.} Equal cumulative weights are optimal when future use is unknown and remain within a $C^2$ factor of the best task-specific weights when future queries are known. Two-layer and GPT-2 experiments support both the width advantage and the predicted failure of unequal weights. 
\end{enumerate} 

\section{Related Work} \label{sec:related} 
\paragraph{Explicit and continuous reasoning.} Chain-of-thought writes intermediate computation as tokens~\citep{wei2022cot,merrill2024cot}, whereas COCONUT feeds latent states back into the model~\citep{hao2024coconut}. Recent theory shows that a single continuous thought can preserve and expand several reasoning traces in parallel~\citep{zhu2025reasoning,gozeten2025continuous}, and subsequent work studies how such superposition emerges during training~\citep{zhu2025emergence}. Zhu et al.~\citep{zhu2025reasoning} establish that frontier superposition can implement parallel search: the latent state keeps the alternatives active at the current step and expands them together. We build on this representation and study its capacity limit over a complete $D$-step Transformer computation. In particular, we compare keeping only the frontier with keeping all reached states, and we determine how the latent state should distribute weight when later queries are not known in advance. 

\paragraph{Capacity of superposition.} Work on feature superposition explains how non-orthogonal features share a low-dimensional space and how interference depends on sparsity, importance, and polysemanticity~\citep{elhage2022toy,foote2024polysemanticity,gupta2025polysemantic,bereska2025lossy}. These analyses ask how many features can be represented and whether individual features remain recoverable. Our setting instead asks whether a normalized latent state can support a later attention computation. When that computation combines several reached states, their aligned components add coherently while unrelated candidates remain random interference. This produces the reversal studied in our theory: preserving more reached states can reduce the hidden width required for correct attention. We further compare all nonnegative state weightings and show when equal cumulative weights are optimal. 

\paragraph{Recurrent computation and latent states.} Universal and looped Transformers reuse computation depth while updating a latent state~\citep{dehghani2019universal,giannou2023looped,xu2024looped,fu2026fullylooped}. Recurrent and buffer-based models similarly carry information through dedicated latent tokens~\citep{bulatov2022recurrent,yang2024bufferthoughts,lin2026tyler}. Related work studies multi-hop reasoning in long contexts and without explicit intermediate steps~\citep{li2024longcontextmultihop,yao2025implicitmultihop}, as well as the limits and internal representations of hidden reasoning~\citep{browncohen2026opaque,hu2024hopfieldian}. We connect these questions to the Transformer operations that update a latent state: attention retrieves earlier states, values carry new information, and the residual path determines what remains available. Our analysis identifies the least represented state as the bottleneck and yields a concrete principle for repeated computation: when future use is unknown, keep reached states equally weighted. 

\section{Problem Setting and Central Hypothesis}
\label{sec:setup}

We study the continuous-thought token of a recurrent two-block Transformer.
Following the copied-edge interface of \citet{zhu2025reasoning}, object
embeddings are available to attention as keys and values.  A recurrent cell
first uses $QK^\top$ attention to decide which objects are active, then writes
their values into the next thought. We use this common cell to compare one proposed state construction against its natural baseline and against all nonnegative history weightings.

Let $G_{\mathrm{graph}}=(V,E)$ be a directed graph with $n=|V|$ and root
$r$.  At a fixed reasoning depth, let
\begin{equation}
F:=\{v:\operatorname{dist}(r,v)=c\},
\qquad
R:=\{v:\operatorname{dist}(r,v)\le c\},
\label{eq:FR}
\end{equation}
with $F\subseteq R$, $|F|=k$, and $|R|=m$, where $1\le k\le m$.  Vertices
represent
intermediate facts, entities, subgoals, partial plans, or candidate states.
Across reasoning steps we write
\begin{equation}
R_c:=\{v:\operatorname{dist}(r,v)\le c\},\qquad
F_c:=\{v:\operatorname{dist}(r,v)=c\},\qquad
N_c:=R_{c+1}\setminus R_c.
\label{eq:reached-sequence}
\end{equation}
When $k$ is the number of distinct states first reached at the current depth, whereas $m$ is
the number first reached up to that depth.  Thus a length-$c$ chain has $k=1$
and $m=c+1$.

Each vertex has an independent Gaussian embedding
\begin{equation}
g_v\sim N(0,I_d/d).
\label{eq:embeddings}
\end{equation}
Equivalently, collect the embeddings as columns of
$G\in\R^{d\times n}$ with independent $N(0,1/d)$ entries.  For a coefficient
vector $a\in\R^n$, write
\begin{equation}
X_a:=Ga=\sum_{v\in V}a_vg_v,
\qquad
q_a:=\frac{X_a}{\norm{X_a}}.
\label{eq:qa}
\end{equation}
Multiplying $a$ by a positive scalar does not change $q_a$, so we use
$\norm{a}_2=1$ throughout.

The proposed state, its baseline, and the comparison family are
\begin{align}
a_F&:=\frac{\mathbf 1_F}{\sqrt{k}},
&a_R&:=\frac{\mathbf 1_R}{\sqrt{m}},
&a&\in\mathcal A_R:=\{a\in\R_+^n:\operatorname{supp}(a)\subseteq R,\ \norm a_2=1\},
\label{eq:frontier-state}\\
q_F&:=q_{a_F},
&q_R&:=q_{a_R},
&q_a&:=\frac{X_a}{\norm{X_a}}.
\label{eq:cumulative-state}
\end{align}
Frontier is the baseline that retains
only current alternatives.  Uniform cumulative is the proposed state.  The
set $\mathcal A_R$ contains every nonnegative history weighting against which
we compare that proposal.

Let $h(a)$ denote the thought token carrying $q_a$, and let $h_j$ be a
candidate token carrying coefficient direction $u_j$.  The final
Transformer head uses
\begin{equation}
W_Qh(a)=q_a,\qquad
W_Kh_j=X_{u_j}=Gu_j,\qquad
W_Vh_j=v_j,
\label{eq:qkv-interface}
\end{equation}
where, independently of $G$,
\begin{equation}
u_0,u_1,\ldots,u_M\in\R^n
\label{eq:bank}
\end{equation}
are unit directions, $u_0$ is useful, $u_1,\ldots,u_M$ compete with it, and
$v_j$ is the value routed by candidate $j$.  Its actual forward pass is
\begin{equation}
\alpha_j(q_a)
:=\frac{\exp(\inner{q_a}{X_{u_j}}/\tau)}
{\sum_{i=0}^M\exp(\inner{q_a}{X_{u_i}}/\tau)},
\qquad
\operatorname{Attn}(q_a)=\sum_{j=0}^M\alpha_j(q_a)v_j,
\label{eq:attention-forward}
\end{equation}
and the selected candidate is
\begin{equation}
\widehat j(q_a)
:=\arg\max_{0\le j\le M}\alpha_j(q_a)
=\arg\max_{0\le j\le M}\inner{q_a}{X_{u_j}}.
\label{eq:decoder}
\end{equation}
Success means that this head assigns its largest attention weight to the
useful value, $\widehat j(q_a)=0$.  A node key is recovered by $u_0=e_v$. The final attention head uses the same $M$ competing keys for every state
weighting.  We assume
\(
u_1,\ldots,u_M\ \text{are orthonormal and orthogonal to }
\operatorname{span}\{a,u_0\},
\label{eq:isotropic-competition}
\)
so no competing key has a built-in alignment advantage.  When comparing
states supported on $R$, we choose the bank orthogonal to the $R$ coordinates,
making it common to all weightings and requiring $M\le n-m$.
Appendix~\ref{app:arbitrary-bank} treats a general fixed bank.

For the comparisons below, the useful direction $u_0$ is
nonnegative and supported on $R$.  The population quantity is
\begin{equation}
\mu(a,u_0):=\inner{a}{u_0}\in[0,1].
\label{eq:mu}
\end{equation}
Nonnegative $a,u_0$ express the positive state components represented by the
superposition construction. Let $\mathcal U$ be the set of possible useful directions, all nonnegative,
unit-normalized, and supported on $R$.  We assume that no direction places
more than $C/\sqrt m$ weight on one reached object:
\begin{equation}
\norm{u}_\infty\le \frac{C}{\sqrt m},
\qquad u\in\mathcal U,
\label{eq:central-query-condition}
\end{equation}
where $C\ge1$. It measures how unevenly a
future computation may use the reached history.  A uniform history query has
$C=1$; bounded variation over a constant fraction of the history has
$C=O(1)$.

\paragraph{Central hypothesis.}
Among all sequences of nonnegative history weights used by the recurrent
Transformer, uniform cumulative's joint width measure should be within $C^2$
of the best sequence chosen with advance knowledge of $\mathcal U$.  If only
condition~\eqref{eq:central-query-condition} is known and the realized future
query is not, uniform cumulative should be the best
single choice.  Sections~\ref{sec:dimension-law} and~\ref{sec:comparison}
prove these statements; Section~\ref{sec:experiments} tests their predicted
ordering and Transformer mechanism.

\section{The Width of a $D$-Step Transformer Computation}
\label{sec:dimension-law}

We measure the width needed to keep a state usable throughout $D$ attention
calls and the final answer. Fix a coefficient sequence
$\boldsymbol a=(a_c)_{c=0}^D$, independently of $G$, with
$a_c\in\mathcal A_{R_c}$. Let $S_c$ contain the sources that must remain
usable at step $c$, and let $I_c\subseteq V\setminus\operatorname{supp}(a_c)$
contain the inactive sources. Define
\begin{equation}
a_{\min,c}:=\min_{v\in S_c}a_{c,v},\qquad
\alpha_\star:=\min_{c<D}a_{\min,c},\qquad
\tau_c:=3a_{\min,c}/8.
\label{eq:amin}
\end{equation}
Frontier uses $S_c=F_c$; full reuse of cumulative states uses $S_c=R_c$.
The intended states $q_{a_c}$ give the joint routing event
\begin{equation}
\mathcal E_{\rm route}(d):=
\bigcap_{c<D}\left\{
\min_{v\in S_c}\inner{q_{a_c}}{g_v}>\tau_c>
\max_{s\in I_c}\inner{q_{a_c}}{g_s}\right\}.
\label{eq:routing-event}
\end{equation}
For a nonempty compact final query class $\mathcal U\subseteq\mathcal A_{R_D}$,
put
\begin{equation}
\Gamma(a_D;\mathcal U):=\inf_{u\in\mathcal U}\inner{a_D}{u}.
\label{eq:Gamma}
\end{equation}
Write $d_D$ for the least width at which the intended states satisfy both
\eqref{eq:routing-event} and the correct final-answer event with probability
at least $1-\delta$ for every fixed $u\in\mathcal U$. Its width scale is
\begin{equation}
\mathfrak D(\boldsymbol a;\mathcal U):=
\max\left\{\max_{c<D}\frac{\log(Dn/\delta)}{a_{\min,c}^2},
\frac{2\log M}{\Gamma(a_D;\mathcal U)^2}\right\}.
\label{eq:weighted-dimension}
\end{equation}
The actual Transformer state is denoted by $\widetilde q_c$. The construction
in Section~\ref{sec:transformer} receives copied edge buffers, fixed masks,
deduplicated new destinations, and their counts. These inputs and all target
coefficients are fixed before $G$ is sampled. The following result connects
the width comparison to repeated attention, residual writing, and RMSNorm.

\begin{theorem}[Width and stability of $D$-step Transformer computation]
\label{thm:transformer-realization}
Fix $D$ and $\delta\in(0,1)$, assume $\alpha_\star>0$ and
$\Gamma:=\Gamma(a_D;\mathcal U)>0$, and use $\Theta(n)$ inactive node keys
at each step and the common final bank. As $M\to\infty$, with $n$ large
enough for the bank, the intended-state joint requirement satisfies
\begin{equation}
d_D(\boldsymbol a;\mathcal U,M,\delta)
=\Theta\!\left(\mathfrak D(\boldsymbol a;\mathcal U)\right).
\label{eq:d-step-width-law}
\end{equation}
In particular, the routing width and final-answer width obey
\begin{equation}
d_{\rm route}=\Theta\!\left(\max_{c<D}
\frac{\log(Dn/\delta)}{a_{\min,c}^2}\right),
\label{eq:transformer-routing-law}
\end{equation}
\begin{equation}
d_{\rm ans}=(2+o(1))\frac{\log M}{\Gamma^2}.
\label{eq:matching-dimension}
\end{equation}
For the cumulative or frontier updates in Section~\ref{sec:transformer},
there are finite attention scales and a constant $K_{D,\delta}$ such that
$d\ge K_{D,\delta}\mathfrak D(\boldsymbol a;\mathcal U)$ makes the
actual $D$-step computation and final answer correct with probability at
least $1-\delta$, for every $u\in\mathcal U$. On the same event,
\begin{equation}
\max_{c\le D}\norm{\widetilde q_c-q_{a_c}}
\le A_D\left(\sqrt{\frac{\log(16(D+1)/\delta)}d}+\zeta\right),
\label{eq:trace-state-error}
\end{equation}
where $\zeta$ bounds each destination-write error and can be made
arbitrarily small by finite attention scales. Constants may depend on the
fixed number $D$ of recurrent steps.
\end{theorem}

\begin{remark}
A state must support both the next attention call and the eventual answer.
The smallest retained coefficient controls the first requirement, while
alignment with the final query controls the second. The stability bound
shows how the same width scale also supports repeated Transformer updates:
writing and normalization errors stay below the attention gaps across the
whole trace, rather than being reset after each step.
\end{remark}

\section{Theoretical Validation of Cumulative State}
\label{sec:comparison}

We now compare complete $D$-step computations through the width scale in
Theorem~\ref{thm:transformer-realization}.  The comparison keeps the graph,
Transformer cell, and final candidate bank fixed; only the coefficients
carried by the thought token change.

\subsection{The $C^2$ guarantee for a known query class}
\label{sec:weighted}

The query condition in Section~\ref{sec:setup} is equivalently
$\mathcal U\subseteq\mathcal U_C(R)$, where
\begin{equation}
\mathcal U_C(R)
:=\left\{u\in\mathcal A_R:\norm{u}_\infty\le C/\sqrt{|R|}\right\};
\label{eq:full-nonlocalized-class}
\end{equation}
the class may be finite or continuous and need not treat history positions
symmetrically.

\begin{theorem}[Cumulative guarantee and best fixed weighting]
\label{thm:minimax}
Suppose every reached object remains eligible for later source-membership
attention.  Let $\mathcal U\subseteq\mathcal U_C(R_D)$ be any nonempty compact
query class, and let $\boldsymbol a_R=(a_{R_c})_{c=0}^D$ be the uniform
cumulative sequence.  Then
\begin{equation}
\mathfrak D(\boldsymbol a_R;\mathcal U)
\le C^2
\inf_{\substack{a_c\in\mathcal A_{R_c}\\0\le c\le D}}
\mathfrak D(\boldsymbol a;\mathcal U).
\label{eq:minimax-cumulative-approx}
\end{equation}
The factor $C^2$ for the width scale $\mathfrak D$ cannot in general be reduced.  If the realized query is not
known beyond the condition defining $\mathcal U_C(R_D)$, cumulative is the
best single weighting for the full allowed family:
\begin{equation}
\boldsymbol a_R
\in
\arg\min_{\substack{a_c\in\mathcal A_{R_c}\\0\le c\le D}}
\mathfrak D(\boldsymbol a;\mathcal U_C(R_D)).
\label{eq:minimax-cumulative}
\end{equation}
\end{theorem}

\begin{remark}
The comparison allows every nonnegative coefficient sequence, including
sequences that a scalar residual update cannot generate. Equal cumulative
weights still lose at most $C^2$ against the best such choice for a known
query class, and minimize the full-family objective when the query is
unknown. The same inequalities hold over any smaller comparison family
containing the cumulative sequence. Together with
Theorem~\ref{thm:transformer-realization}, this gives an actual cumulative
implementation at width at most $K_{D,\delta}C^2\inf_{\boldsymbol a}
\mathfrak D(\boldsymbol a;\mathcal U)$: the optimal weight comparison and
the stability of its Transformer implementation are both controlled.
\end{remark}

\subsection{Why preserving history can reduce width}
\label{sec:mechanism}

The query condition in this subsection reads
\begin{equation}
\norm{u}_\infty\le\frac{C}{\sqrt m},
\qquad u\in\mathcal U.
\label{eq:query-nonlocalization}
\end{equation}

\paragraph{Tasks covered by the condition.}
Since $\norm{u}_2=1$, a fixed $C$
means that every query draws on at least $m/C^2$ reached objects.  Uniform
history readouts have $C=1$, while bounded-imbalance aggregation over a
constant fraction has $C=O(1)$.  This covers whole-set aggregation, additive
objectives over visited objects, and supervision over algorithmic traces
\citep{zaheer2017deepsets,lee2019settransformer,kool2019routing,
velickovic2022clrs,lightman2024verify}.  The regime $k=o(m)$ includes chains,
fixed-width traces, and grid-like breadth-first search: the current frontier
grows more slowly than the latent states back into computation.

\begin{theorem}[When cumulative reduces the width of the complete trace]
\label{thm:uniform-separation}
Fix $F\subsetneq R$, with $|F|=k<m=|R|$, and use the same candidate bank for
both states.  Let $\mathcal U\subseteq\mathcal A_R$ be any nonempty compact
query class satisfying~\eqref{eq:query-nonlocalization}.  Then
\begin{equation}
\Gamma(a_R;\mathcal U)\ge\frac1C,
\qquad
\Gamma(a_F;\mathcal U)\le C\sqrt{\frac{k}{m}},
\label{eq:nonlocalized-alignment-gap}
\end{equation}
Let
\begin{equation}
T_R:=\frac{2\log M}{\Gamma(a_R;\mathcal U)^2},
\qquad
T_F:=\frac{2\log M}{\Gamma(a_F;\mathcal U)^2}
\label{eq:terminal-width-terms}
\end{equation}
be the final-answer terms of the two $D$-step computations.  Whenever
$\Gamma(a_F;\mathcal U)>0$,
\begin{equation}
\frac{T_R}{T_F}
\le C^4\frac{k}{m}.
\label{eq:uniform-separation}
\end{equation}
Writing $R_R,R_F$ for their routing terms in
\eqref{eq:weighted-dimension}, their complete width scales are
$\max\{R_R,T_R\}$ and $\max\{R_F,T_F\}$.  Hence the same ratio describes the
end-to-end dimensions whenever final answer selection is the active
bottleneck for both policies.  For fixed $C$ it then vanishes in the
growing-history regime $k=o(m)$.  The rate $C^4k/m$ is sharp: it is
attained when $k\le m/C^2$ and $m/C^2$ is an integer, and asymptotically
attained along fixed-$C$ growing-history sequences with $m/C^2\to\infty$.  If
$\Gamma(a_F;\mathcal U)=0$, frontier's final head cannot exceed chance at any
finite dimension while cumulative still can.
\end{theorem}

\begin{remark}
This theorem identifies when the end-to-end gain comes from the final use of
state.  As the reached history grows, cumulative strengthens the correct
answer logit even though it stores more objects.  Equation~\eqref{eq:weighted-dimension}
also shows when intermediate routing, rather than final readout, sets the
required width; the $C^2$ result above covers both bottlenecks simultaneously.
\end{remark}

For the representative history-wide direction $u=a_R$, the final terms are
exactly
\begin{equation}
T_F=\left(2\frac{m}{k}+o\!\left(\frac{m}{k}\right)\right)\log M,
\qquad
T_R=(2+o(1))\log M.
\label{eq:mk-example}
\end{equation}
Thus the cumulative final state uses a factor $m/k$ less width for this
history-wide computation; Theorem~\ref{thm:transformer-realization} combines
this gain with the cost of producing all preceding states.

The reason is coherent addition.  Readout succeeds when
\begin{equation}
\underbrace{\inner a u\sqrt d}_{\text{useful signal}}
\gtrsim
\underbrace{\sqrt{2\log M}}_{\text{unrelated competition}}.
\label{eq:signal-noise}
\end{equation}
For cumulative, $\inner{a_R}{u}=\norm{u}_1/\sqrt m$: useful historical
components add before normalization, while unrelated candidates remain
centered fluctuations.  The extra stored objects therefore strengthen a
history-wide computation instead of acting as independent noise.  This is the
mathematical source of the apparent reversal.

\subsection{How the same Transformer writes frontier or cumulative}
\label{sec:transformer}

Each recurrent call receives the thought, a threshold token, and copied
source--destination buffers. The fixed interface supplies one active edge
per newly reached destination in $N_c$, removes already reached destinations,
and supplies $m_c=|R_c|$ and $r_c=|N_c|$. Thus destinations are counted once;
every active source lies in $S_c$ and every inactive source lies in $I_c$.
These masks, copies, and counts are inputs to the construction. We consider
$r_c\ge1$; a trace stops when no new destination remains.

For an edge $e=(s,t)$, the first block uses
\begin{equation}
Q_1h_e=[g_s;1],\qquad K_1h_c=[\widetilde q_c;0],\qquad
K_1h_{\tau_c}=[0;\tau_c].
\label{eq:block1-qk}
\end{equation}
The edge attends only to the thought and threshold tokens. Their values
$1$ and $0$ give the gate
\begin{equation}
z_e=\sigma\!\left(\beta(\inner{\widetilde q_c}{g_s}-\tau_c)\right).
\label{eq:block1-gate}
\end{equation}
The second block reads the gate and copies destination values:
\begin{equation}
Q_2h_{c+1}=1,\qquad K_2h_e=T(2z_e-1),\qquad V_2h_e=g_t.
\label{eq:block2-qkv}
\end{equation}
After multiplying its attention average by $\sqrt{r_c}$, its output satisfies
\begin{equation}
\widetilde o_c=\frac1{\sqrt{r_c}}\sum_{t\in N_c}g_t+\xi_c,
\qquad \norm{\xi_c}\le\zeta.
\label{eq:destination-write}
\end{equation}
A masked copy head returns $\widetilde q_c$. Prescribed scalar gates
$\lambda_c,\eta_c$ scale the copy and destination branches; a fixed output
projection combines them with the residual. RMSNorm, with gain chosen to
return unit state norm, gives
\begin{equation}
\widetilde q_{c+1}
=\operatorname{RMSNorm}(\lambda_c\widetilde q_c+\eta_c\widetilde o_c).
\label{eq:transformer-update}
\end{equation}
Here $d$ is the node/key channel width; a constant number of buffer channels
uses $O(d)$ hidden width. The scalar controls include the supplied counts.
The two choices are
\begin{equation}
(\lambda_c,\eta_c)=(0,1)\quad\text{(frontier)},\qquad
(\lambda_c,\eta_c)=(\sqrt{m_c},\sqrt{r_c})\quad\text{(cumulative)}.
\label{eq:two-updates}
\end{equation}
Starting from the normalized root state, the latter approximates uniform
cumulative states with the error in~\eqref{eq:trace-state-error}. Other scalar
gates produce weights by arrival batch: they rescale all retained
coefficients together and assign a common coefficient to the new batch.
The comparison in Theorem~\ref{thm:minimax} allows the larger family of all
nonnegative weight sequences. Cumulative's guarantee therefore holds even
against choices beyond those generated by these scalar gates.

\section{Experimental Evidence}
\label{sec:experiments}

We test the cumulative-superposition hypothesis at the same three levels as
the theory. Within every comparison, the task, candidates, embedding codebook, and trained
attention mechanism are shared; only the imposed state coefficients change.

\subsection{Test I: Does Cumulative Need Less Width?}
\label{sec:exp-e1}

We test Theorem~\ref{thm:uniform-separation} on a grid over
$m\in\{64,128,256,512\}$, $k/m\in\{1/32,\ldots,1/2\}$, $C\in\{1,1.25,1.5,2\}$,
and $M\in\{2^{10},2^{14}\}$, estimating worst-case $D_\star$ by binary search
on $d$ with $2000$ Monte Carlo trials ($\delta=0.1$).
Each query is drawn independently of $G$ and projected onto the cap
$\norm u_\infty\le C/\sqrt m$.
\emph{Dense} uses random positive weights on all $m$ coordinates, modeling
uneven aggregation over a visited set \citep{zaheer2017deepsets}.
\emph{Temporal} centers an exponential bump along the trace, modeling
localized process supervision over a CoT \citep{lightman2024verify}.
\emph{Block} partitions the trace into segments with different scales,
modeling staged algorithmic readouts \citep{velickovic2022clrs}.
\emph{Union} mixes the three generators.
The \emph{sharp witness} is the theorem's example, uniform on
$\lceil m/C^2\rceil$ coordinates that contain $F$.
\emph{Zero-overlap} is uniform on $R\setminus F$, so $\Gamma(a_F;\mathcal U)=0$.

\begin{figure}[htbp]
\centering
\includegraphics[width=\linewidth]{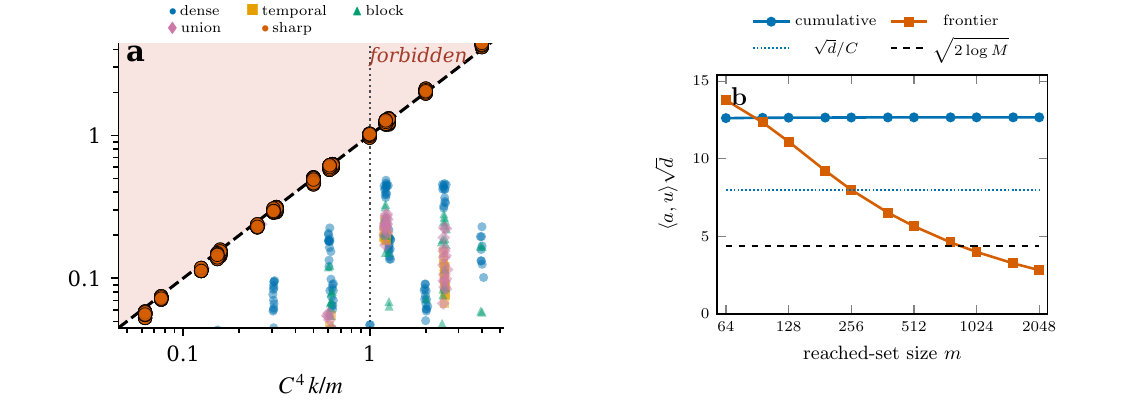}
\caption{\textbf{Why remembering more can require less width.}
(a) Across many ways of using the reached history state, cumulative state stays below the
predicted width bound, while the deliberately hardest examples approach it.
(b) As the reached history grows, cumulative keeps useful signal for the final task,
whereas frontier discards earlier states and becomes harder to use.}
\label{fig:e1}
\end{figure}

\paragraph{Protocol.}
We use the ProsQA\citep{zhong2022proqa} graph-state task and hold the graph trace, candidate bank,
and Gaussian codebook fixed when comparing state constructions.  Queries are
drawn before the Gaussian map. The controlled attention experiment uses
an identity-regularized source-membership router at $d=192$ and evaluates five
independent Gaussian codebooks.

\subsection{Test II: Do Trained Transformers Show the Same Advantage?}
\label{sec:exp-cumulative}
Across the paired five-seed sweeps, cumulative crosses the same 90\% held-out
success threshold at substantially lower width on both architectures.  Its
interpolated minimum reliable width is $40.3\pm1.1$, compared with
$110.6\pm1.6$ for frontier, in the lightweight two-block model; with the
two-layer GPT-2 backbone it is $34.4\pm3.8$, compared with
$66.2\pm5.9$.  The corresponding cumulative/frontier ratios are
$0.365\pm0.008$ and $0.521\pm0.067$.  Thus the separation predicted by
Theorem~\ref{thm:uniform-separation} persists when the analytic state construction is replaced by a trained causal
Transformer.  Figure~\ref{fig:frontier-cumulative} shows both width sweeps and
their paired minimum-width estimates under matched evaluation protocols.

\begin{figure*}[t]
  \centering
  \begin{minipage}[t]{0.45\textwidth}
    \vspace{0pt}
    \centering
    \includegraphics[width=\linewidth]{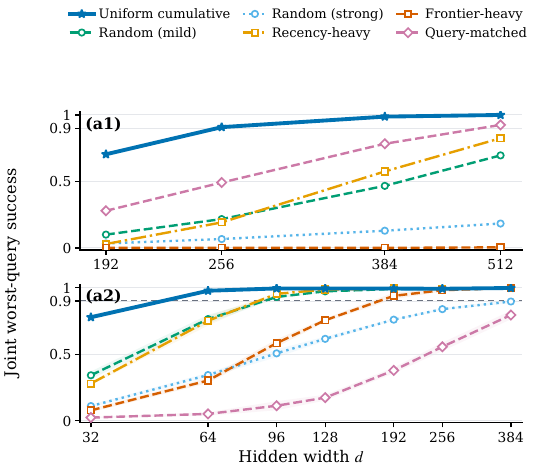}
  \end{minipage}\hfill
  \begin{minipage}[t][0.355\textwidth][s]{0.53\textwidth}
    \vspace{1.2mm}
    \centering
    \textbf{(b)\quad Paired five-seed results}
    \vspace{0.8mm}

    \scriptsize
    \setlength{\tabcolsep}{4.0pt}
    \renewcommand{\arraystretch}{1.12}
    \resizebox{\linewidth}{!}{%
    \begin{tabular}{lccc}
      \toprule
      \multicolumn{4}{l}{\textit{Lightweight two-block model}} \\
      State & Acc. @ 32 & Acc. @ 40 & Empirical $d_{0.9}$ \\
      \midrule
      Frontier & $0.210\pm0.007$ & $0.322\pm0.005$ & $110.6\pm1.6$ \\
      \textbf{Cumulative} & $\mathbf{0.793\pm0.015}$ & $\mathbf{0.899\pm0.010}$ & $\mathbf{40.3\pm1.1}$ \\
      Ideal frontier & $0.216\pm0.007$ & $0.328\pm0.006$ & $112$ \\
      Ideal cumulative & $0.947\pm0.005$ & $0.990\pm0.002$ & $32$ \\
      \midrule
      \multicolumn{4}{l}{\textit{GPT-2 backbone}} \\
      State & Acc. @ 32 & Acc. @ 64 & Empirical $d_{0.9}$ \\
      \midrule
      Frontier & $0.468\pm0.038$ & $0.900\pm0.026$ & $66.2\pm5.9$ \\
      \textbf{Cumulative} & $\mathbf{0.898\pm0.024}$ & $\mathbf{0.998\pm0.003}$ & $\mathbf{34.4\pm3.8}$ \\
      \bottomrule
    \end{tabular}%
    }

    \vfill
    \setlength{\tabcolsep}{5.2pt}
    \resizebox{\linewidth}{!}{%
    \begin{tabular}{lcc}
      \toprule
      Width ratio $d^{\rm C}_{0.9}/d^{\rm F}_{0.9}$ $\downarrow$
        & Two-layer & GPT-2 \\
      \midrule
      Paired estimate & $0.365\pm0.008$ & $0.521\pm0.067$ \\
      \bottomrule
    \end{tabular}%
    }

  \end{minipage}
  \caption{\textbf{Trained Transformers reproduce the predicted width advantage.}
The shared legend shows six ways to weight the same reached history.
(a1) Results for the two-layer block model; (a2) results for the GPT-2-block
model, with bands showing variation across runs. Every curve in a panel uses the
same trained attention module, and equal cumulative weighting reaches $90\%$
accuracy on the hardest tested task first. (b) Five paired runs confirm that
cumulative reaches the same held-out accuracy at a smaller width than frontier in
both architectures. Here $d_{0.9}$ means the smallest interpolated width reaching
$90\%$ accuracy; Ideal supplies the intended state directly.}
  \label{fig:frontier-cumulative}
\end{figure*}

\subsection{Test III: What Fails Under Unequal Weights?}
\label{sec:exp-weighted}

We next test the mechanism behind Theorem~\ref{thm:minimax}. Figure~\ref{fig:weighted-mechanism}(a) shows the predicted two-term
geometry: uniform cumulative is the joint lower-left point after normalizing
routing and worst-query readout width by the uniform construction.  Panels
(b,c) expose what this optimum means inside the router.  Uniform weights
equalize the coefficient of every reusable state token; recency and random
weights produce low-coefficient tokens, and those tokens receive lower
membership margins.  Consequently, joint route success falls from $0.906$ to
$0.586$ and $0.537$, while the per-inactive-key false-positive rate rises from
$0.03\%$ to $0.57\%$ and $1.31\%$, respectively.

\begin{figure*}[t]
  \centering
  \includegraphics[width=\textwidth]{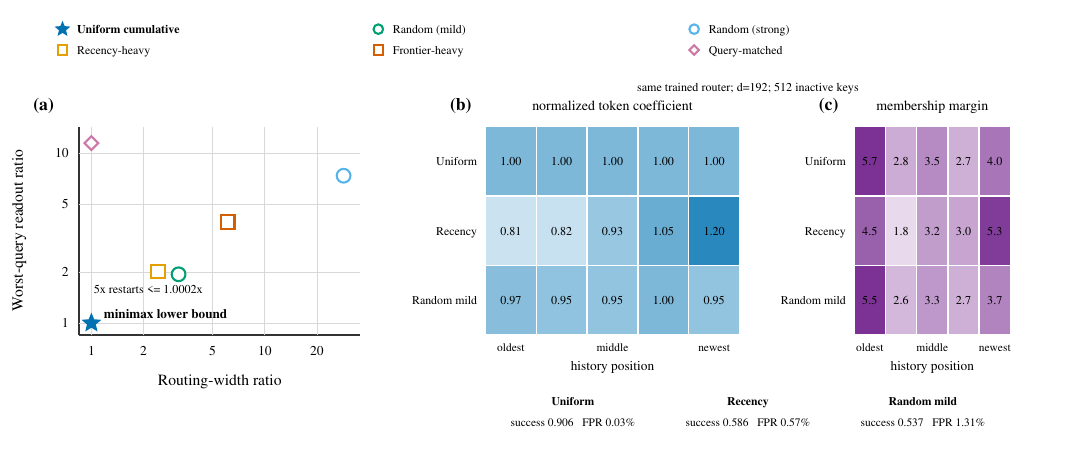}
  \caption{\textbf{Why equal cumulative weights are safest.}
(a) Each point compares the width needed to keep states usable during reasoning and
the width needed to use them for the final answer; lower is better.
(b) The heatmap shows how much state is assigned to each position in the reached
history. (c) Using the same trained attention module, states with small weights are
the first to be missed. The summary reports success on the whole routing step and
the rate at which inactive states are selected by mistake.}
  \label{fig:weighted-mechanism}
\end{figure*}

\section{Conclusion}
\label{sec:conclusion}

We asked how reasoning should be organized inside a Transformer's latent
space. Our answer is the cumulative-superposition hypothesis: when past
reasoning may be useful again, the latent state should preserve it with equal
weights instead of keeping only the latest frontier.

The theory explains why this can help. Useful parts of the history reinforce
one another in later attention, so cumulative state can reach reliable choices
at a smaller hidden width than frontier state. This advantage can survive
multi-step Transformer computation: attention finds relevant states, values
carry new information, the residual keeps earlier information, and the
feed-forward layer restores equal weights after each write. Finally, when the
model does not know which part of state a future task will use, equal weights
give the strongest guarantee; when the task family is known, they remain within
the explicit $C^2$ factor of the best specialized weighting.

The experiments match these predictions. Both the lightweight and GPT-2-block
Transformers reach the same accuracy at smaller width with cumulative state.
The position-level tests also show the predicted mechanism: unequally weighted states
fail first where they assign the least weight.

Together, the results turn cumulative superposition into a concrete design
principle for continuous reasoning: preserve the reached states evenly and
restore that balance as computation proceeds. This allows a fixed hidden
width to support longer reusable histories and more reliable future
computation.

\newpage
\subsection*{AI use statement}

Generative AI tools were only used for grammar correction and writing optimization, as well as generating some chart-drawing code, and all AI-assisted content was carefully reviewed by the author. The author takes full responsibility for the final content of this work, including its statements, proofs, experiments, citations, charts, and text.

\bibliography{iclr2027_conference}

@inproceedings{zhong2022proqa,
    title = "{P}ro{QA}: Structural Prompt-based Pre-training for Unified Question Answering",
    author = "Zhong, Wanjun  and
      Gao, Yifan  and
      Ding, Ning  and
      Qin, Yujia  and
      Liu, Zhiyuan  and
      Zhou, Ming  and
      Wang, Jiahai  and
      Yin, Jian  and
      Duan, Nan",
    editor = "Carpuat, Marine  and
      de Marneffe, Marie-Catherine  and
      Meza Ruiz, Ivan Vladimir",
    booktitle = "Proceedings of the 2022 Conference of the North American Chapter of the Association for Computational Linguistics: Human Language Technologies",
    month = jul,
    year = "2022",
    address = "Seattle, United States",
    publisher = "Association for Computational Linguistics",
    url = "https://aclanthology.org/2022.naacl-main.313/",
    doi = "10.18653/v1/2022.naacl-main.313",
    pages = "4230--4243",
}

@inproceedings{wei2022cot,
  title={Chain-of-Thought Prompting Elicits Reasoning in Large Language Models},
  author={Wei, Jason and Wang, Xuezhi and Schuurmans, Dale and Bosma, Maarten and Ichter, Brian and Xia, Fei and Chi, Ed H. and Le, Quoc V. and Zhou, Denny},
  booktitle={Advances in Neural Information Processing Systems},
  volume={35},
  year={2022}
}

@inproceedings{dehghani2019universal,
  title={Universal Transformers},
  author={Dehghani, Mostafa and Gouws, Stephan and Vinyals, Oriol and Uszkoreit, Jakob and Kaiser, Lukasz},
  booktitle={International Conference on Learning Representations},
  year={2019}
}

@inproceedings{giannou2023looped,
  title={Looped Transformers as Programmable Computers},
  author={Giannou, Angeliki and Rajput, Shashank and Sohn, Jy-Yong and Lee, Kangwook and Lee, Jason D. and Papailiopoulos, Dimitris},
  booktitle={Proceedings of the 40th International Conference on Machine Learning},
  series={Proceedings of Machine Learning Research},
  volume={202},
  pages={11398--11442},
  year={2023}
}

@inproceedings{bulatov2022recurrent,
  title={Recurrent Memory Transformer},
  author={Bulatov, Aydar and Kuratov, Yuri and Burtsev, Mikhail},
  booktitle={Advances in Neural Information Processing Systems},
  volume={35},
  year={2022}
}

@article{hao2024coconut,
  title={Training Large Language Models to Reason in a Continuous Latent Space},
  author={Hao, Shibo and Sukhbaatar, Sainbayar and Su, DiJia and Li, Xian and Hu, Zhiting and Weston, Jason and Tian, Yuandong},
  journal={arXiv preprint arXiv:2412.06769},
  year={2024}
}

@inproceedings{zhu2025reasoning,
  title={Reasoning by Superposition: A Theoretical Perspective on Chain of Continuous Thought},
  author={Zhu, Hanlin and Hao, Shibo and Hu, Zhiting and Jiao, Jiantao and Russell, Stuart and Tian, Yuandong},
  booktitle={Advances in Neural Information Processing Systems},
  year={2025},
  note={arXiv:2505.12514}
}

@article{zhu2025emergence,
  title={Emergence of Superposition: Unveiling the Training Dynamics of Chain of Continuous Thought},
  author={Zhu, Hanlin and Hao, Shibo and Hu, Zhiting and Jiao, Jiantao and Russell, Stuart and Tian, Yuandong},
  journal={arXiv preprint arXiv:2509.23365},
  year={2025}
}

@article{gozeten2025continuous,
  title={Continuous Chain of Thought Enables Parallel Exploration and Reasoning},
  author={Gozeten, Halil Alperen and Ildiz, M. Emrullah and Zhang, Xuechen and Harutyunyan, Hrayr and Rawat, Ankit Singh and Oymak, Samet},
  journal={arXiv preprint arXiv:2505.23648},
  year={2025}
}

@inproceedings{merrill2024cot,
  title={The Expressive Power of Transformers with Chain of Thought},
  author={Merrill, William and Sabharwal, Ashish},
  booktitle={International Conference on Learning Representations},
  year={2024}
}

@article{elhage2022toy,
  title={Toy Models of Superposition},
  author={Elhage, Nelson and Hume, Tristan and Olsson, Catherine and Schiefer, Nicholas and Henighan, Tom and Kravec, Shauna and Hatfield-Dodds, Zac and Lasenby, Robert and Drain, Dawn and Chen, Carol and others},
  journal={Transformer Circuits Thread},
  year={2022},
  url={https://transformer-circuits.pub/2022/toy_model/index.html}
}

@article{xu2024looped,
  title={On Expressive Power of Looped Transformers: Theoretical Analysis and Enhancement via Timestep Encoding},
  author={Xu, Kevin and Sato, Issei},
  journal={arXiv preprint arXiv:2410.01405},
  year={2024}
}

@article{fu2026fullylooped,
  title={Simply Stabilizing the Loop via Fully Looped Transformer},
  author={Fu, Rao and Yang, Zixuan and Zhang, Jiankun and Ma, Jing and Chen, Hechang and Li, Yu and Chang, Yi},
  journal={arXiv preprint arXiv:2605.18797},
  year={2026}
}

@article{yao2025implicitmultihop,
  title={Language models can learn implicit multi-hop reasoning, but only if they have lots of training data},
  author={Yao, Yuekun and Du, Yupei and Zhu, Dawei and Hahn, Michael and Koller, Alexander},
  journal={arXiv preprint arXiv:2505.17923},
  year={2025}
}

@article{li2024longcontextmultihop,
  title={Making Long-Context Language Models Better Multi-Hop Reasoners},
  author={Li, Yanyang and Liang, Shuo and Lyu, Michael R. and Wang, Liwei},
  journal={arXiv preprint arXiv:2408.03246},
  year={2024}
}

@article{yang2024bufferthoughts,
  title={Buffer of Thoughts: Thought-Augmented Reasoning with Large Language Models},
  author={Yang, Ling and Yu, Zhaochen and Zhang, Tianjun and Cao, Shiyi and Xu, Minkai and Zhang, Wentao and Gonzalez, Joseph E. and Cui, Bin},
  journal={arXiv preprint arXiv:2406.04271},
  year={2024}
}

@article{lin2026tyler,
  title={Tyler: Typed Latent Reasoning for Language Models -- When to Think, What to Compute, and How Much to Allocate},
  author={Lin, Hanyu and Cai, Min and Wen, Jiawei and Zhang, Haodi},
  journal={arXiv preprint arXiv:2606.16360},
  year={2026}
}

@article{browncohen2026opaque,
  title={Quantifying the Necessity of Chain of Thought through Opaque Serial Depth},
  author={Brown-Cohen, Jonah and Lindner, David and Shah, Rohin},
  journal={arXiv preprint arXiv:2603.09786},
  year={2026}
}

@article{hu2024hopfieldian,
  title={Understanding Reasoning in Chain-of-Thought from the Hopfieldian View},
  author={Hu, Lijie and Liu, Liang and Yang, Shu and Chen, Xin and Tan, Zhen and Ali, Muhammad Asif and Li, Mengdi and Wang, Di},
  journal={arXiv preprint arXiv:2410.03595},
  year={2024}
}

@article{foote2024polysemanticity,
  title={Tackling Polysemanticity with Neuron Embeddings},
  author={Foote, Alex},
  journal={arXiv preprint arXiv:2411.08166},
  year={2024}
}

@article{gupta2025polysemantic,
  title={Disentangling Polysemantic Neurons with a Null-Calibrated Polysemanticity Index and Causal Patch Interventions},
  author={Gupta, Manan and Kumar, Dhruv},
  journal={arXiv preprint arXiv:2508.16950},
  year={2025}
}

@article{bereska2025lossy,
  title={Superposition as Lossy Compression: Measure with Sparse Autoencoders and Connect to Adversarial Vulnerability},
  author={Bereska, Leonard and Tzifa-Kratira, Zoe and Samavi, Reza and Gavves, Efstratios},
  journal={arXiv preprint arXiv:2512.13568},
  year={2025}
}

@inproceedings{zaheer2017deepsets,
  title={Deep Sets},
  author={Zaheer, Manzil and Kottur, Satwik and Ravanbakhsh, Siamak and Poczos, Barnabas and Salakhutdinov, Ruslan and Smola, Alexander J.},
  booktitle={Advances in Neural Information Processing Systems},
  volume={30},
  year={2017}
}

@inproceedings{lee2019settransformer,
  title={Set Transformer: A Framework for Attention-based Permutation-Invariant Neural Networks},
  author={Lee, Juho and Lee, Yoonho and Kim, Jungtaek and Kosiorek, Adam and Choi, Seungjin and Teh, Yee Whye},
  booktitle={Proceedings of the 36th International Conference on Machine Learning},
  pages={3744--3753},
  year={2019}
}

@inproceedings{kool2019routing,
  title={Attention, Learn to Solve Routing Problems!},
  author={Kool, Wouter and van Hoof, Herke and Welling, Max},
  booktitle={International Conference on Learning Representations},
  year={2019}
}

@inproceedings{velickovic2022clrs,
  title={The {CLRS} Algorithmic Reasoning Benchmark},
  author={Veli{\v{c}}kovi{\'c}, Petar and Badia, Adri{\`a} Puigdom{\`e}nech and Budden, David and Pascanu, Razvan and Banino, Andrea and Dashevskiy, Misha and Hadsell, Raia and Blundell, Charles},
  booktitle={Proceedings of the 39th International Conference on Machine Learning},
  pages={22084--22102},
  year={2022}
}

@inproceedings{lightman2024verify,
  title={Let's Verify Step by Step},
  author={Lightman, Hunter and Kosaraju, Vineet and Burda, Yura and Edwards, Harri and Baker, Bowen and Lee, Teddy and Leike, Jan and Schulman, John and Sutskever, Ilya and Cobbe, Karl},
  booktitle={International Conference on Learning Representations},
  year={2024}
}
\bibliographystyle{iclr2027_conference}

\appendix

\section{Proofs for the Dimension Limit}
\label{app:dimension}

The complete-trace theorem uses the following local attention facts.  They are
stated here because an individual call is a proof component rather than the
object compared in the main text.

For this appendix only, let
\begin{equation}
d_\star(a,u_0;M,\delta)
:=\min\{d:\Pr(\widehat j(q_a)=0)\ge1-\delta\},
\label{eq:critical-dim}
\end{equation}
with $d_\star=\infty$ when $\mu(a,u_0)=0$ and
$1-\delta>1/(M+1)$.

\begin{lemma}[Exact distribution of one attention comparison]
\label{lem:exact-scores}
Fix unit $a,u_0$ with $\mu:=\inner a{u_0}\in[0,1]$, and assume
\eqref{eq:isotropic-competition}.  There exist mutually independent
$Z_0,Z_1,\ldots,Z_M\sim N(0,1)$ and $C_d\sim\chi_d$ such that
\begin{align}
\sqrt d\,\inner{q_a}{X_{u_0}}
&\overset{d}{=}
\mu C_d+\sqrt{1-\mu^2}\,Z_0,
\label{eq:true-score}\\
\sqrt d\,\inner{q_a}{X_{u_j}}
&\overset{d}{=}Z_j,
\qquad 1\le j\le M.
\label{eq:decoy-score}
\end{align}
Consequently,
\begin{equation}
\Pr\!\left(\alpha_0(q_a)=\max_{0\le j\le M}\alpha_j(q_a)\right)
=
\Pr\left(
\mu C_d+\sqrt{1-\mu^2}\,Z_0>\max_{1\le j\le M}Z_j
\right).
\label{eq:exact-success}
\end{equation}
\end{lemma}

\begin{lemma}[Monotonicity in alignment]
\label{lem:alignment-monotonicity}
For fixed $d\ge1$ and $M\ge1$, the probability in
\eqref{eq:exact-success} is strictly increasing in $\mu\in[0,1]$.
\end{lemma}

\begin{theorem}[Local attention transition]
\label{thm:phase-transition}
If $M\to\infty$ and
$d_M\mu_M^2/(2\log M)\to\lambda$, then
\begin{equation}
\Pr\!\left(\alpha_0(q_a)=\max_{0\le j\le M}\alpha_j(q_a)\right)
\longrightarrow
\begin{cases}
0,&\lambda<1,\\
1,&\lambda>1.
\end{cases}
\label{eq:zero-one}
\end{equation}
\end{theorem}

\begin{corollary}[Dimension of one attention comparison]
\label{cor:critical-dimension}
For fixed $\delta\in(0,1)$,
\begin{equation}
d_\star(a,u_0;M,\delta)
=(2+o(1))\frac{\log M}{\mu_M^2}.
\label{eq:single-head-dimension}
\end{equation}
\end{corollary}

\begin{proof}[Proof of Lemma~\ref{lem:exact-scores}]
Let $Y:=Ga$.  Since $\norm a_2=1$, $Y\sim N(0,I_d/d)$ and
$\sqrt d\norm Y\sim\chi_d$.

If $\mu<1$, define
\begin{equation}
c:=\frac{u_0-\mu a}{\sqrt{1-\mu^2}}.
\label{eq:cdef}
\end{equation}
Then $c$ is unit norm, $c\perp a$, and
$u_0=\mu a+\sqrt{1-\mu^2}c$.  If $\mu=1$, the term involving $c$ below is
absent.  Gaussianity and $c\perp a$ imply that $Gc$ is independent of $Y$.
Conditioned on $Y$, rotational invariance gives
\begin{equation}
\sqrt d\inner{Y/\norm Y}{Gc}\sim N(0,1),
\label{eq:conditional-c}
\end{equation}
and this random variable is independent of $\norm Y$.
Therefore
\begin{align}
\sqrt d\inner{q_a}{Gu_0}
&=
\mu\sqrt d\norm Y
+\sqrt{1-\mu^2}\sqrt d\inner{Y/\norm Y}{Gc}
\nonumber\\
&\overset d=
\mu C_d+\sqrt{1-\mu^2}Z_0.
\label{eq:proof-true}
\end{align}

For each $j\ge1$, $u_j\perp a$, so $Gu_j$ is independent of $Y$ and
\begin{equation}
\sqrt d\inner{Y/\norm Y}{Gu_j}\sim N(0,1).
\label{eq:proof-decoy}
\end{equation}
The orthogonality of $u_1,\ldots,u_M$ and their orthogonality to $u_0$
imply joint independence of all the Gaussian terms.  Comparing the scores
proves~\eqref{eq:exact-success}.
\end{proof}

\begin{proof}[Proof of Lemma~\ref{lem:alignment-monotonicity}]
Let $S_\mu:=\mu C_d+\sqrt{1-\mu^2}Z$ and
$F_\mu(t):=\Pr(S_\mu\le t)$.  For $0<\mu<1$, put
$\sigma^2:=1-\mu^2$.  Differentiating under the integral against the
$\chi_d$ density gives
\begin{equation}
\frac{\partial F_\mu(t)}{\partial\mu}
=-K_{d,\mu,t}
\int_0^\infty(c-\mu t)c^{d-1}
\exp\left(-\frac{(c-\mu t)^2}{2\sigma^2}\right)dc,
\label{eq:monotonicity-derivative}
\end{equation}
where $K_{d,\mu,t}>0$.  Denote the integral by $I$.  Since
$(c-\mu t)e^{-(c-\mu t)^2/(2\sigma^2)}$ equals $-\sigma^2$ times the
derivative of the exponential, integration by parts yields
\begin{equation}
I=
\begin{cases}
\sigma^2\exp\left(-\dfrac{\mu^2t^2}{2\sigma^2}\right),&d=1,\\[5pt]
\sigma^2(d-1)\displaystyle\int_0^\infty c^{d-2}
\exp\left(-\dfrac{(c-\mu t)^2}{2\sigma^2}\right)dc,&d>1.
\end{cases}
\label{eq:positive-monotonicity-integral}
\end{equation}
Thus $I>0$ and $\partial_\mu F_\mu(t)<0$ for every $t$.  Continuity handles
$\mu=0,1$, so $S_\mu$ is strictly increasing in first-order stochastic order.
Conditional on $S_\mu=s$, success against $M$ independent standard-normal
decoys has probability $\Phi(s)^M$, a strictly increasing function of $s$.
Therefore the unconditional success probability is strictly increasing in
$\mu$.  The critical-dimension conclusion follows immediately.
\end{proof}

\begin{lemma}[Gaussian maximum and chi concentration]
\label{lem:extremes}
Let $Z_1,\ldots,Z_M$ be independent $N(0,1)$ variables and let
$C_d\sim\chi_d$.  As $M,d\to\infty$,
\begin{equation}
\frac{\max_{j\le M}Z_j}{\sqrt{2\log M}}\to1,
\qquad
\frac{C_d}{\sqrt d}\to1
\label{eq:extreme-limits}
\end{equation}
in probability.
\end{lemma}

\begin{proof}
For the upper Gaussian bound, for every $\varepsilon>0$, the union bound and
$\Pr(Z\ge t)\le e^{-t^2/2}$ give
\begin{align}
\Pr\left(\max_{j\le M}Z_j>(1+\varepsilon)\sqrt{2\log M}\right)
&\le
M\exp(-(1+\varepsilon)^2\log M)
\nonumber\\
&=M^{-2\varepsilon-\varepsilon^2}\to0.
\end{align}
For the lower bound, let $t=(1-\varepsilon)\sqrt{2\log M}$.  The Gaussian
Mills lower bound gives, for a universal $c>0$ and all sufficiently large
$t$,
\begin{equation}
\Pr(Z>t)\ge c\frac{e^{-t^2/2}}{t}.
\end{equation}
Therefore $M\Pr(Z>t)\to\infty$, and independence yields
\begin{equation}
\Pr\left(\max_{j\le M}Z_j\le t\right)
=(1-\Pr(Z>t))^M
\le e^{-M\Pr(Z>t)}\to0.
\end{equation}
This proves the first convergence.  For the second,
$C_d^2/d$ is the average of $d$ independent squared standard Gaussians, so it
converges to one in probability by the weak law of large numbers.  The square
root is continuous on $(0,\infty)$.
\end{proof}

\begin{proof}[Proof of Theorem~\ref{thm:phase-transition}]
Divide the two sides of~\eqref{eq:exact-success} by
$\sqrt{2\log M}$.  Lemma~\ref{lem:extremes} gives
\begin{equation}
\frac{\max_{j\le M}Z_j}{\sqrt{2\log M}}\to1.
\end{equation}
If $\lambda=0$, then
\begin{equation}
\E\left[
\left(\frac{\mu_M C_{d_M}}{\sqrt{2\log M}}\right)^2
\right]
=\frac{d_M\mu_M^2}{2\log M}\to0,
\end{equation}
so the aligned term converges to zero in probability without requiring
$d_M\to\infty$.

Now suppose $\lambda\in(0,\infty]$.  Since $\mu_M\le1$, the assumed ratio
implies $d_M\to\infty$, and
The aligned term satisfies
\begin{align}
\frac{\mu_M C_{d_M}}{\sqrt{2\log M}}
&=
\sqrt{\frac{d_M\mu_M^2}{2\log M}}
\frac{C_{d_M}}{\sqrt{d_M}}
\to\sqrt\lambda,
\label{eq:aligned-limit}
\end{align}
where convergence to $+\infty$ is understood when $\lambda=\infty$.
Because $|\sqrt{1-\mu_M^2}|\le1$ and a standard Gaussian is tight,
\begin{equation}
\frac{\sqrt{1-\mu_M^2}Z_0}{\sqrt{2\log M}}\to0
\end{equation}
in probability.  If $1<\lambda<\infty$, choose a constant strictly between
$1$ and $\sqrt\lambda$; the true score lies above that constant and the
maximum decoy below it with probability tending to one.  If
$0<\lambda<1$, choose a constant strictly between $\sqrt\lambda$ and $1$
and reverse the argument.  The cases $\lambda=0$ and $\lambda=\infty$
follow from the two endpoint limits above.  This proves~\eqref{eq:zero-one}.
\end{proof}

\paragraph{Additional consequence for the number of candidates.}
Let $N:=n-m$ and define
\begin{equation}
M_{\max}(d,\mu,\delta,N)
:=\max\{1\le M\le N:p_{d,M}(\mu)\ge1-\delta\},
\label{eq:Mmax-definition}
\end{equation}
where $p_{d,M}(\mu)$ is the probability in
Equation~\eqref{eq:exact-success}.  If
$\min\{d\mu^2,\log N\}\to\infty$, then
\begin{equation}
\log M_{\max}
=(1+o(1))\min\left\{\frac{d\mu^2}{2},\log N\right\}.
\label{eq:capacity-exponent}
\end{equation}

\begin{proof}[Proof of Corollary~\ref{cor:critical-dimension}]
For fixed $M$ and $\mu>0$, the success probability in
\eqref{eq:exact-success} is nondecreasing in $d$.  To see this, couple
$C_{d+1}=\sqrt{C_d^2+Z^2}$ with an additional independent standard Gaussian
$Z$; then $C_{d+1}\ge C_d$ almost surely, while all other variables in
\eqref{eq:exact-success} are unchanged.

Fix $\varepsilon>0$.  Theorem~\ref{thm:phase-transition} implies that at
dimensions
\begin{equation}
d_-=(1-\varepsilon)\frac{2\log M}{\mu_M^2},
\qquad
d_+=(1+\varepsilon)\frac{2\log M}{\mu_M^2},
\end{equation}
the success probabilities tend to zero and one, respectively.  For all
sufficiently large $M$, the critical dimension lies between $d_-$ and $d_+$.
Letting $\varepsilon\downarrow0$ proves~\eqref{eq:single-head-dimension}.
Solving the same transition relation for $\log M$ gives
\eqref{eq:capacity-exponent}.
\end{proof}

\section{Proofs for Frontier, Cumulative, and Weighted States}
\label{app:structures}

\begin{proof}[Proof of Theorem~\ref{thm:uniform-separation}]
Fix $u\in\mathcal U$.  Nonnegativity, unit norm, and
\eqref{eq:query-nonlocalization} give
\begin{equation}
1=\norm u_2^2
\le \norm u_\infty\norm u_1
\le \frac{C}{\sqrt m}\norm u_1.
\end{equation}
Therefore
\begin{equation}
\inner{a_R}{u}=\frac{\norm u_1}{\sqrt m}\ge\frac1C.
\end{equation}
For frontier,
\begin{equation}
\inner{a_F}{u}
=\frac1{\sqrt k}\sum_{v\in F}u_v
\le\sqrt k\norm u_\infty
\le C\sqrt{\frac{k}{m}}.
\end{equation}
Taking infima proves~\eqref{eq:nonlocalized-alignment-gap}.

Compactness and Lemma~\ref{lem:alignment-monotonicity} imply that the
worst-query success probability under state $a$ is exactly
$p_{d,M}(\Gamma(a;\mathcal U))$.  If
$\Gamma(a_F;\mathcal U)>0$, the uniform critical-dimension law gives
\begin{equation}
\frac{D_\star(a_R;\mathcal U,M,\delta)}
{D_\star(a_F;\mathcal U,M,\delta)}
=
\frac{\Gamma(a_F;\mathcal U)^2}
{\Gamma(a_R;\mathcal U)^2}(1+o(1))
\le C^4\frac{k}{m}(1+o(1)).
\end{equation}
The remainder is uniform over positive alignments by the asymptotic convention.
If $\Gamma(a_F;\mathcal U)=0$, compactness supplies a query with zero frontier
alignment; its correct answer is exchangeable with the $M$ decoys, giving the
stated infinite critical dimension.

It remains to prove sharpness.  Suppose first that $s=m/C^2$ is an integer and
$k\le s$.  Choose $S\subseteq R$ with $|S|=s$ and $F\subseteq S$, and let
$u_v=C/\sqrt m$ on $S$ and zero elsewhere.  Then $u$ is a nonnegative unit
vector satisfying~\eqref{eq:query-nonlocalization}.  For the compact singleton
class $\mathcal U=\{u\}$,
\begin{equation}
\Gamma(a_R;\mathcal U)=\frac1C,
\qquad
\Gamma(a_F;\mathcal U)=C\sqrt{\frac{k}{m}},
\end{equation}
so equality holds in~\eqref{eq:uniform-separation}.  If $m/C^2$ is not an
integer, use $\lfloor m/C^2\rfloor$ coordinates at the cap and one residual
coordinate.  The rounding changes the ratio by $1+o(1)$ whenever
$m/C^2\to\infty$.  The condition $k\le m/C^2$ holds eventually in every
fixed-$C$ growing-history regime with $k=o(m)$.
\end{proof}

Equation~\eqref{eq:mk-example} follows by taking $u=a_R$: then
$\inner{a_R}{u}=1$ and $\inner{a_F}{u}=\sqrt{k/m}$, and substitution into
Corollary~\ref{cor:critical-dimension} gives the two dimensions.

\begin{proof}[Proof of Theorem~\ref{thm:transformer-realization}]
All target coefficients, graph sets, masks, and candidate directions below
are fixed independently of $G$. We first bound attention gaps for these
intended states, then transfer the gaps to the actual recurrent computation.
The probability statement is uniform in the choice of a fixed
$u\in\mathcal U$, not a simultaneous event over all queries.

\paragraph{Intended-state attention gaps and width.}
At step $c$, set $a=a_c$ and $\alpha=a_{\min,c}$. Gaussian regression gives,
for $v\in S_c$,
\begin{equation}
\sqrt d\,\inner{q_a}{g_v}
\overset d=a_vC_d+\sqrt{1-a_v^2}\,Z_v,
\qquad C_d=\sqrt d\norm{Ga}\sim\chi_d,
\label{eq:membership-score}
\end{equation}
where each $Z_v$ is marginally standard normal. For $s\in I_c$,
\begin{equation}
\sqrt d\,\inner{q_a}{g_s}\sim N(0,1).
\label{eq:inactive-membership-score}
\end{equation}
These inactive scores are independent within the step and independent of
the intended state and required-source scores. The required-source scores
need not be mutually independent. On the event
\begin{equation}
C_d/\sqrt d\ge3/4,\qquad
\min_{v\in S_c}Z_v\ge-\alpha\sqrt d/4,\qquad
\max_{s\in I_c}\sqrt d\inner{q_a}{g_s}\le\alpha\sqrt d/4,
\label{eq:routing-good-event}
\end{equation}
required scores are at least $\alpha/2$, inactive scores are at most
$\alpha/4$, and the threshold $3\alpha/8$ has margin $\alpha/8$.
Gaussian and chi-square tails imply, for absolute constants $A,b>0$,
\begin{equation}
\Pr(\mathcal E_{\rm route}(d)^c)
\le A\sum_{c<D}n\exp(-bd\,a_{\min,c}^2).
\label{eq:routing-tail}
\end{equation}
The same bound applies to failure of the stronger margin event just used.
For a lower bound, successful routing requires the weakest source at each
step to beat $\Theta(n)$ inactive keys. Its score law is
Lemma~\ref{lem:exact-scores} with alignment $a_{\min,c}$, so
Corollary~\ref{cor:critical-dimension} gives
$\Omega(\log n/a_{\min,c}^2)$. Since $D,\delta$ are fixed, these lower
bounds and~\eqref{eq:routing-tail} prove~\eqref{eq:transformer-routing-law}.

For any fixed final query $u$, Lemma~\ref{lem:exact-scores} applies
unconditionally to the deterministic coefficients $a_D,u$. Monotonicity
in alignment and compactness give the worst-query alignment
$\Gamma=\Gamma(a_D;\mathcal U)$. Consequently
\begin{equation}
d_{\rm ans}=(2+o(1))\frac{\log M}{\Gamma^2}.
\label{eq:appendix-final-width}
\end{equation}
Here and in~\eqref{eq:matching-dimension}, $M\to\infty$ at fixed failure
probability, with the remainder uniform over positive alignments as in
Corollary~\ref{cor:critical-dimension}. Allocating half the error to routing
and half to the final answer proves the upper bound for $d_D$. Joint success
requires each event separately, proving its lower bound and hence
\eqref{eq:d-step-width-law}. No independence between steps is used.

\paragraph{A common event for the actual computation.}
Fix $u\in\mathcal U$ and consider either the cumulative or frontier target
sequence. Set $\rho=\min(\alpha_\star,\Gamma)$. For all $c\le D$ require
\begin{equation}
\big|\norm{Ga_c}-1\big|\le\epsilon,
\qquad \epsilon=K\sqrt{\log(16(D+1)/\delta)/d}\le1/4.
\label{eq:update-norm-event}
\end{equation}
Also require every node key, $Gu$, and every competing final key to have
norm at most $2$. Gaussian norm concentration and a union bound ensure
these events with total failure probability at most $\delta/2$, provided
$d\ge K\log((n+M+D+1)/\delta)$ and $K$ is sufficiently large.
Assign probability $\delta/4$ to failure of the route-margin events above.
For the final head, with $\mu=\inner{a_D}{u}\ge\Gamma$, require
\begin{equation}
\inner{q_{a_D}}{Gu}\ge\Gamma/2,
\qquad \max_{j\ge1}\inner{q_{a_D}}{Gu_j}\le\Gamma/4.
\label{eq:final-margin-event}
\end{equation}
Indeed, in the exact score law, $C_d/\sqrt d\ge3/4$ and
$Z_0\ge-\Gamma\sqrt d/4$ imply the first inequality; the usual Gaussian
maximum bound implies the second. Their total failure probability is at
most $\delta/4$ when $d\ge K\log((M+1)/\delta)/\Gamma^2$.
The intersection thus has probability at least $1-\delta$. It is built from
fixed coefficient directions before considering the actual updates.
All subsequent arguments are deterministic on this intersection.

\paragraph{Finite attention gates and destination writes.}
Let $e_c=\norm{\widetilde q_c-q_{a_c}}$ and assume $e_c\le\rho/32$.
Every source score changes by at most $2e_c$, leaving threshold margin at
least $\alpha_\star/16$. Thus~\eqref{eq:block1-gate} has active gates in
$[1-h,1]$ and inactive gates in $[0,h]$, where
$h=\exp(-\beta\alpha_\star/16)$. Let $L$ bound the number of edge tokens
per call and let $r_{\max}=\max_c r_c$. With $h\le1/4$, the inactive mass
of the destination head is at most $Le^{-T}$. Among active edges, its logits
vary by at most $2Th$, so its conditional distribution has $\ell_1$
distance at most $e^{2Th}-1$ from the uniform distribution. For $2Th\le1$,
this distance is at most $4Th$. Since every destination value has norm
at most $2$, after output scaling,
\begin{equation}
\left\|\widetilde o_c-G\frac{\mathbf1_{N_c}}{\sqrt{r_c}}\right\|
\le8\sqrt{r_{\max}}Th+4\sqrt{r_{\max}}Le^{-T}.
\label{eq:destination-head-error}
\end{equation}
This uses one active edge per new destination. For $0<\zeta<1$, choose
\begin{equation}
T\ge\max\{1,\log(8\sqrt{r_{\max}}L/\zeta)\},\qquad
h\le\min\{1/4,1/(2T),\zeta/(16\sqrt{r_{\max}}T)\}.
\label{eq:softmax-leakage}
\end{equation}
A finite $\beta$ achieves this bound on $h$; both scales are chosen from
the deterministic counts and tolerances, independently of $G$. Hence the
vector error in~\eqref{eq:destination-write} is at most $\zeta$.

\paragraph{Residual writing and normalization across all steps.}
The masked copy head is exact. The prescribed scalar branches and fixed
output projection write $(\lambda_c-1)\widetilde q_c+
\eta_c\widetilde o_c$, which combines with the residual as
\eqref{eq:transformer-update}. RMSNorm is used with fixed gain $1/\sqrt d$
on the $d$-dimensional state channel and zero normalization offset, so it
returns the unit direction.
For cumulative, write $m_c=|R_c|$, $b_c=\mathbf1_{N_c}/\sqrt{r_c}$.
Disjointness gives the exact coefficient identity
\begin{equation}
a_{R_{c+1}}=
\frac{\sqrt{m_c}a_{R_c}+\sqrt{r_c}b_c}{\sqrt{m_c+r_c}}.
\end{equation}
Compare the actual pre-normalization state
$\widetilde z_c=\sqrt{m_c}\widetilde q_c+\sqrt{r_c}\widetilde o_c$
with $y_c=\sqrt{m_c+r_c}\,Ga_{R_{c+1}}$. By
\eqref{eq:update-norm-event} and~\eqref{eq:destination-head-error},
\begin{equation}
\frac{\norm{\widetilde z_c-y_c}}{\sqrt{m_c+r_c}}
\le e_c+\epsilon+\zeta,
\qquad
\frac{\norm{y_c}}{\sqrt{m_c+r_c}}\ge1-\epsilon.
\label{eq:cumulative-direction-error}
\end{equation}
For nonzero $x,y$, the normalization inequality
$\norm{x/\norm x-y/\norm y}\le2\norm{x-y}/\norm y$
therefore yields
\begin{equation}
e_{c+1}\le4e_c+4\epsilon+4\zeta.
\label{eq:update-error-recursion}
\end{equation}
The small-error choices below ensure $\widetilde z_c\ne0$.
For frontier, $b_c=a_{F_{c+1}}$ and the old branch is discarded, giving
$e_{c+1}\le4\zeta$, which also satisfies~\eqref{eq:update-error-recursion}.
Initialize $\widetilde q_0=q_{a_0}$. With
$H_D=4(4^D-1)/3$, induction gives $e_c\le H_D(\epsilon+\zeta)$.
Choose $\epsilon,\zeta\le\rho/(128H_D)$. Then $e_c\le\rho/64$ at every
step, which closes the gate-margin induction. Absorbing the absolute
constant in $\epsilon$ into $A_D$ proves~\eqref{eq:trace-state-error}.

\paragraph{Final attention and sufficient width.}
The perturbation of any final score is at most $2e_D$ on the key-norm event.
Thus the ideal gap $\Gamma/4$ in~\eqref{eq:final-margin-event} remains
strictly positive: its reduction is at most $4e_D\le\Gamma/16$.
The actual final head selects the useful key. This is a pathwise transfer,
not a Gaussian calculation conditional on $\widetilde q_D$.
All concentration and induction requirements hold under
\begin{equation}
d\ge K_D\max\left\{
\frac{\log(Dn/\delta)}{\alpha_\star^2},
\frac{\log((M+1)/\delta)}{\Gamma^2},
\frac{\log(16(D+1)/\delta)}{\rho^2}\right\}.
\end{equation}
For fixed $D,\delta$ and $M\ge2$, the last term and the key-norm bound are
absorbed by a constant multiple of~\eqref{eq:weighted-dimension}.
This proves the actual-computation sufficient-width claim.
The constant tracks the accumulation factor $H_D$; the theorem makes its
asymptotic comparison at fixed $D$.
\end{proof}

\begin{proof}[Proof of Theorem~\ref{thm:minimax}]
Write $L:=\log(Dn/\delta)$ and $B:=2\log M$, and interpret a zero
alignment or reusable coefficient as infinite width.  For every $c<D$ and
every $a_c\in\mathcal A_{R_c}$,
\begin{equation}
\left(\min_{v\in R_c}a_{c,v}\right)^2
\le \frac1{|R_c|}\sum_{v\in R_c}a_{c,v}^2
=\frac1{|R_c|},
\label{eq:min-coordinate}
\end{equation}
with equality only at $a_c=a_{R_c}$.  Hence the routing part of
$\mathfrak D(\boldsymbol a_R;\mathcal U)$ is no larger than that of any
$\mathfrak D(\boldsymbol a;\mathcal U)$.

For the readout part, Theorem~\ref{thm:uniform-separation} gives
$\Gamma(a_{R_D};\mathcal U)\ge1/C$.  Cauchy--Schwarz gives
$\Gamma(a_D;\mathcal U)\le1$ for every unit $a_D$, so
\begin{equation}
\frac{B}{\Gamma(a_{R_D};\mathcal U)^2}
\le C^2B
\le C^2\frac{B}{\Gamma(a_D;\mathcal U)^2}.
\label{eq:minimax-readout-factor}
\end{equation}
Combining the routing inequality and~\eqref{eq:minimax-readout-factor}, and
using $C\ge1$, yields
\begin{equation}
\mathfrak D(\boldsymbol a_R;\mathcal U)
\le C^2\mathfrak D(\boldsymbol a;\mathcal U)
\end{equation}
for every weighted sequence $\boldsymbol a$.  Taking the infimum proves
\eqref{eq:minimax-cumulative-approx}.

The factor is sharp.  Suppose $s:=|R_D|/C^2$ is an integer, choose
$S\subseteq R_D$ with $|S|=s$, and let
$u=\mathbf1_S/\sqrt s$ and $\mathcal U=\{u\}$.  This class is
$C$-nonlocalized.  Cumulative has terminal alignment $1/C$, whereas the
weighted terminal state $a_D=u$ has alignment one.  Keep all earlier states
cumulative.  Whenever the readout scale $B$ is at least the cumulative routing
scale $L\max_{c<D}|R_c|$, the two objectives are respectively $C^2B$ and $B$,
so equality holds in~\eqref{eq:minimax-cumulative-approx}.
Such instances are compatible with the common-bank constraint: take $D=1$,
$|R_0|=1$, and enough unused vertices that $M$ is proportional to $n$.
For nonintegral $m/C^2$, cap $\lfloor m/C^2\rfloor$ coordinates and use
one residual coordinate; as $m\to\infty$, the ratio tends to $C^2$.

It remains to prove exact robust optimality.  Put $m:=|R_D|$ and let
$u_\star$ minimize $\norm u_1$ over the nonempty compact set
$\mathcal U_C(R_D)$.  This full class is closed under every coordinate
permutation $P$ of $R_D$.  Hence, for any $a_D\in\mathcal A_{R_D}$,
\begin{align}
\Gamma(a_D;\mathcal U_C(R_D))
&\le \frac1{m!}\sum_P\inner{a_D}{Pu_\star}
=\frac{\norm{a_D}_1\norm{u_\star}_1}{m}\notag\\
&\le\frac{\norm{u_\star}_1}{\sqrt m}
=\Gamma(a_{R_D};\mathcal U_C(R_D)).
\label{eq:robust-terminal-optimum}
\end{align}
The second inequality uses $\norm{a_D}_1\le\sqrt m\norm{a_D}_2=\sqrt m$;
the final equality follows because
$\inner{a_{R_D}}u=\norm u_1/\sqrt m$.  Thus $a_{R_D}$ minimizes the robust
readout term.  Equation~\eqref{eq:min-coordinate} shows that every earlier
$a_{R_c}$ minimizes its routing term.  The cumulative sequence therefore
minimizes their maximum, proving~\eqref{eq:minimax-cumulative}.

Finally, restricting the comparison to any family containing
$\boldsymbol a_R$ can only increase its infimum, so the same $C^2$ bound
continues to hold. Since the full-family minimum is attained by
$\boldsymbol a_R$, its exact minimax statement also survives this
restriction. The actual cumulative construction has sufficient width
$K_{D,\delta}\mathfrak D(\boldsymbol a_R;\mathcal U)$ by
Theorem~\ref{thm:transformer-realization}. Combining these two inequalities
gives the implementation guarantee stated in the remark. The precise
factor $C^2$ concerns the defined objective $\mathfrak D$; the construction's
stability constant is accounted for separately.
\end{proof}

\section{Exact Reduction for an Arbitrary Fixed Bank}
\label{app:arbitrary-bank}

The orthogonal bank makes the main dimension law transparent.  The following
identity shows that the same signal-versus-competition structure persists for
an arbitrary fixed correlated bank.

\begin{theorem}[Arbitrary-bank Gaussian reduction]
\label{thm:arbitrary-bank}
Fix a unit state $a$ and deterministic candidate directions
$u_0,u_1,\ldots,u_M$, independent of $G$.  For each $j\ge1$, let
\begin{equation}
h_j:=u_0-u_j,
\qquad
\gamma_j:=\inner a{h_j}>0,
\qquad
r_j:=h_j-\gamma_j a\in a^\perp.
\label{eq:arbitrary-decomp}
\end{equation}
Let $g\sim N(0,I_n)$ and $C_d\sim\chi_d$ be independent.  Then
\begin{equation}
\Pr(\widehat j(q_a)=0)
=
\Pr\left(
\max_{1\le j\le M}
\inner g{-r_j/\gamma_j}
<C_d
\right).
\label{eq:arbitrary-exact}
\end{equation}
\end{theorem}

\begin{proof}
The correct candidate beats candidate $j$ exactly when
\begin{align}
0
&<\inner{Ga}{G(u_0-u_j)}
\nonumber\\
&=\gamma_j\norm{Ga}^2+\inner{Ga}{Gr_j}.
\end{align}
Since $\gamma_j>0$, simultaneous success is equivalent to
\begin{equation}
\max_{j\le M}\inner{Ga}{G(-r_j/\gamma_j)}<\norm{Ga}^2.
\label{eq:conditional-event}
\end{equation}
Every $r_j$ is orthogonal to $a$.  Joint Gaussianity therefore makes $Ga$
independent of the family $\{G(-r_j/\gamma_j)\}_{j\le M}$.  Conditional on
$Y:=Ga$, the process on the left of~\eqref{eq:conditional-event} has the same
law as
\begin{equation}
\frac{\norm Y}{\sqrt d}
\left\{
\inner g{-r_j/\gamma_j}
\right\}_{j\le M}.
\end{equation}
Canceling $\norm Y>0$ and using
$\sqrt d\norm Y\sim\chi_d$ proves~\eqref{eq:arbitrary-exact}.
\end{proof}

Theorem~\ref{thm:arbitrary-bank} requires no semantic choice of query and no
orthogonality.  Correlation changes only the maximum of the fixed Gaussian
family.  In the isotropic model, that maximum is
$(1/\mu)\max_{j\le M}Z_j$ up to a common centered Gaussian term, recovering
Theorem~\ref{thm:phase-transition}.  More generally, the critical dimension is
the squared typical scale of the Gaussian maximum in
\eqref{eq:arbitrary-exact}.  The frontier--cumulative comparison remains a
comparison under the same bank; the simplified $\log M/\mu^2$ form is obtained
when the competition geometry is common to both states.

\section{Experimental Details and Additional Results}
\label{app:experimental-details}

\subsection{Dataset and query construction}

We use the public graph-4 ProsQA splits released with
\citet{zhu2025reasoning}.  Table~\ref{tab:prosqa-statistics} records the full
split statistics.  The paired width sweep uses every training and test trace
and a fixed 120-trace validation subset.

\begin{table}[ht]
\centering
\small
\caption{ProsQA graph-4 statistics.  Graph sizes and solution lengths are
averaged over traces.}
\label{tab:prosqa-statistics}
\begin{tabular}{lrrrr}
\toprule
Split & \# traces & $|V|$ & $|E|$ & solution length \\
\midrule
Train & 14,785 & 22.8 & 36.5 & 3.5 \\
Validation & 257 & 22.7 & 36.3 & 3.5 \\
Test & 419 & 22.7 & 36.0 & 3.5 \\
\bottomrule
\end{tabular}
\end{table}

For the lightweight two-block sweep, queries are sampled in coefficient space
before the Gaussian embedding.  Each query contains the complete terminal
frontier but remains strictly partial,
uses between $75\%$ and $85.7\%$ of the reached history, and has within-support
coefficient ratio at most $\rho=1.2$.  These choices imply the certified cap
$C=1.386$; the largest realized caps are smaller on every split.  The same
query, Gaussian draw, useful candidate, and decoy bank are reused for frontier
and cumulative.

\begin{table}[ht]
\centering
\small
\caption{Audit of the fixed nonlocalized query bank.  ``Queries'' counts
coefficient-space queries before repeated Gaussian test draws.}
\label{tab:query-audit}
\begin{tabular}{lrrrr}
\toprule
Split & Used traces & Queries & Support fraction & Realized $C_{\max}$ \\
\midrule
Train & 14,785 & 118,280 & $[0.750,0.857]$ & 1.319 \\
Validation & 120 & 480 & $[0.750,0.857]$ & 1.271 \\
Test & 419 & 3,352 & $[0.750,0.857]$ & 1.287 \\
\bottomrule
\end{tabular}
\end{table}

\subsection{Lightweight two-block width sweep}

The learned construction is the two-attention router in
Theorem~\ref{thm:transformer-realization}.  Its source query and key projections
are initialized as $I+0.08Z/\sqrt d$; the membership inverse temperature and
destination temperature are initialized to $8$ and $10$ and trained jointly.
Frontier and cumulative share all router parameters and optimizer updates.
The best epoch is selected by the sum of their validation state cosines.

\begin{table}[ht]
\centering
\small
\caption{Configuration of the lightweight two-block paired experiment.}
\label{tab:formal-config}
\begin{tabular}{ll}
\toprule
Component & Configuration \\
\midrule
Widths & $16,24,32,40,48,56,64,80,96,112,128$ \\
Training seeds / embedding seeds & $0{:}4$ / $5100{:}5104$ \\
Query-bank size per trace (train/val/test) & $8/4/8$ \\
Decoys (train/val/test) & $256/1{,}024/4{,}096$ \\
Gaussian draws (val/test) & $1/4$ \\
Epochs and optimizer & 4; AdamW, learning rate $3\times10^{-3}$ \\
Regularization & weight decay $10^{-4}$; gradient clip $1.0$ \\
Gradient accumulation & 16 traces \\
Loss weights (gate/state/query) & $1/2/0.1$ \\
Query seed and certified cap & 1701; $C=1.385640646$ \\
Primary endpoint & linearly interpolated $d_{0.9}$ at test success $0.9$ \\
\bottomrule
\end{tabular}
\end{table}

At test time, success is top-1 retrieval of the fixed useful candidate against
4,096 decoys.  Each training seed is evaluated on 419 traces, eight queries per
trace, and four Gaussian draws.  Statistical intervals are therefore computed
across the five independent training seeds, rather than by treating the
within-seed evaluations as independent replications.  During training, each
trace uses one query per epoch, cycling deterministically through the
pre-generated eight-query bank; hence four epochs expose four query IDs per
trace.

\subsection{Complete fixed-query results}

Table~\ref{tab:full-width-results} gives the complete width sweep underlying
Figure~\ref{fig:frontier-cumulative}.  Intervals are two-sided 95\%
Student-$t$ intervals across five training seeds.  ``Ideal'' supplies the
analytic state to the same decoder and is not an additional trained model.

\begin{table}[!ht]
\centering
\small
\caption{Held-out top-1 accuracy at every tested width.}
\label{tab:full-width-results}
\begin{tabular}{rcccc}
\toprule
$d$ & Learned frontier & Learned cumulative & Ideal frontier & Ideal cumulative \\
\midrule
16  & $0.048\pm0.001$ & $0.254\pm0.012$ & $0.052\pm0.001$ & $0.442\pm0.003$ \\
24  & $0.114\pm0.004$ & $0.574\pm0.008$ & $0.120\pm0.005$ & $0.790\pm0.004$ \\
32  & $0.210\pm0.007$ & $0.793\pm0.015$ & $0.216\pm0.007$ & $0.947\pm0.005$ \\
40  & $0.322\pm0.005$ & $0.899\pm0.010$ & $0.328\pm0.006$ & $0.990\pm0.002$ \\
48  & $0.430\pm0.008$ & $0.946\pm0.004$ & $0.436\pm0.008$ & $0.9987\pm0.0002$ \\
56  & $0.533\pm0.010$ & $0.972\pm0.005$ & $0.540\pm0.009$ & $0.9999\pm0.0002$ \\
64  & $0.622\pm0.008$ & $0.984\pm0.003$ & $0.630\pm0.008$ & $1.0000\pm0.0001$ \\
80  & $0.758\pm0.010$ & $0.992\pm0.002$ & $0.767\pm0.010$ & $1.000\pm0.000$ \\
96  & $0.847\pm0.005$ & $0.997\pm0.001$ & $0.855\pm0.006$ & $1.000\pm0.000$ \\
112 & $0.905\pm0.006$ & $0.998\pm0.001$ & $0.913\pm0.005$ & $1.000\pm0.000$ \\
128 & $0.941\pm0.005$ & $0.998\pm0.001$ & $0.949\pm0.003$ & $1.000\pm0.000$ \\
\bottomrule
\end{tabular}
\end{table}

The threshold is stable at the individual-seed level
(Table~\ref{tab:critical-width-seeds}).  Linear interpolation gives a learned
cumulative/frontier ratio of $0.365\pm0.008$ (mean and 95\% interval); the
fixed-alignment prediction is
$(\mu_F/\mu_R)^2=(0.5011/0.8893)^2=0.318$.  The analytic states first cross
the tested grid at widths 112 and 32 for frontier and cumulative, respectively.
At $d=40$, learned-state cosines
are $0.966\pm0.006$ for frontier and $0.933\pm0.006$ for cumulative.  Thus the
cumulative accuracy advantage cannot be explained by a higher state cosine.

\begin{table}[ht]
\centering
\small
\caption{Interpolated empirical critical dimension $d_{0.9}$ for every
lightweight-model training seed.}
\label{tab:critical-width-seeds}
\begin{tabular}{rrrr}
\toprule
Seed & Frontier & Cumulative & C/F \\
\midrule
0 & 108.9 & 39.2 & 0.360 \\
1 & 112.1 & 39.9 & 0.356 \\
2 & 110.3 & 40.4 & 0.366 \\
3 & 111.7 & 41.6 & 0.372 \\
4 & 110.1 & 40.6 & 0.369 \\
\midrule
Mean & 110.6 & 40.3 & 0.365 \\
95\% interval & $\pm1.6$ & $\pm1.1$ & $\pm0.008$ \\
\bottomrule
\end{tabular}
\end{table}

thm:uniform-separation
\subsection{GPT-2 replication: configuration and complete results}

The GPT-2 replication uses the same ProsQA graph-4 traces and the same paired
comparison, but replaces the lightweight router with a two-layer causal GPT-2
Transformer.  Each edge is represented by causal source and destination
tokens followed by one continuous-thought token.  The two states share the
model, optimizer updates, coefficient-space query, Gaussian codebook, useful
candidate, and decoy bank; only the frontier versus cumulative residual update
differs.  Queries are generated before the Gaussian map, use the full reached
support ($\alpha=1$), and draw positive coefficients with imbalance at most
$\rho=1.1$.

\begin{table}[ht]
\centering
\small
\caption{Complete configuration of the paired GPT-2 width sweep.}
\label{tab:gpt2-config}
\renewcommand{\arraystretch}{1.08}
\begin{tabular}{@{}p{0.39\columnwidth}p{0.55\columnwidth}@{}}
\toprule
Component & Configuration \\
\midrule
Backbone & causal GPT-2; 2 layers, 8 heads, FFN width $4d$ \\
Transformer details & GELU, context 512, dropout 0.1, residuals and LayerNorm \\
Tested widths / maximum codebook width & $32,64,96,128,192,256,384$ / 768 \\
Model seeds / embedding seeds & $0{:}4$ / $6100{:}6104$ \\
Used traces (train/validation/test) & $800/120/120$ \\
Queries per trace (train/validation/test) & $8/4/4$ \\
Query construction & full reached support; $\rho=1.1$; seed 2701 \\
Decoys (train/validation/test) & $256/1{,}024/4{,}096$ \\
Gaussian draws (validation/test) & $1/2$ \\
Optimizer & AdamW; learning rate $10^{-4}$; weight decay $10^{-2}$ \\
Optimization & clip 1.0; accumulation 16; at most 40 epochs \\
Loss weights (routing/state/query) & $1/10/0.1$ \\
Stopping rule & after epoch 6: routing $\ge .98$ and both cosines $\ge .80$ \\
Stopping patience / checkpoint selection & 2 epochs / best summed validation cosine \\
Decoder and endpoint & shared top-1 decoder; interpolated $d_{0.9}$ \\
\bottomrule
\end{tabular}
\end{table}

Each seed--width model is evaluated on 120 held-out traces, four fixed queries
per trace, and two Gaussian draws, with the useful candidate ranked against
4,096 fixed decoys.  Table~\ref{tab:gpt2-full-width-results} reports all tested
widths; intervals are two-sided 95\% Student-$t$ intervals over the five model
seeds.

\begin{table}[ht]
\centering
\small
\caption{Complete GPT-2 held-out top-1 accuracy sweep.}
\label{tab:gpt2-full-width-results}
\begin{tabular}{rcc}
\toprule
$d$ & Frontier & Cumulative \\
\midrule
32  & $0.468\pm0.038$ & $0.898\pm0.024$ \\
64  & $0.900\pm0.026$ & $0.998\pm0.003$ \\
96  & $0.978\pm0.009$ & $1.000\pm0.000$ \\
128 & $0.993\pm0.004$ & $1.000\pm0.000$ \\
192 & $0.998\pm0.002$ & $1.000\pm0.000$ \\
256 & $1.000\pm0.001$ & $1.000\pm0.000$ \\
384 & $1.000\pm0.000$ & $1.000\pm0.000$ \\
\bottomrule
\end{tabular}
\end{table}

The seed-level interpolants in Table~\ref{tab:gpt2-critical-width-seeds} are
the raw values summarized in the main-text figure.  Cumulative has a lower
critical dimension for every paired seed; the mean ratio is
$0.521\pm0.067$.

\begin{table}[htbp]
\centering
\small
\caption{GPT-2 interpolated critical dimension $d_{0.9}$ for every model seed.}
\label{tab:gpt2-critical-width-seeds}
\begin{tabular}{rrrr}
\toprule
Seed & Frontier & Cumulative & C/F \\
\midrule
0 & 63.2 & 32.0 & 0.507 \\
1 & 68.6 & 33.0 & 0.481 \\
2 & 73.5 & 36.2 & 0.493 \\
3 & 62.6 & 32.0 & 0.511 \\
4 & 63.0 & 38.8 & 0.616 \\
\midrule
Mean & 66.2 & 34.4 & 0.521 \\
95\% interval & $\pm5.9$ & $\pm3.8$ & $\pm0.067$ \\
\bottomrule
\end{tabular}
\end{table}

\subsection{Weighted routing protocol and raw codebook results}

At $d=192$, the identity-regularized router is trained for four epochs on 600
traces and six profiles with 64 Gaussian route decoys, AdamW
($\mathrm{lr}=3\times10^{-3}$, weight decay $10^{-4}$), clip $1.0$, accumulation
16, initialization noise $0.02$, and identity penalty $10$.  We select by mean
route success minus $0.02$ projection error on 80 validation traces, 256
Gaussian decoys, and one draw.  The frozen checkpoint is evaluated without
retraining on five codebooks, 120 test traces, two draws, all unreached graph
keys, and 512 additional Gaussian decoys per route step.

\begin{table}[htbp]
\centering
\small
\caption{Raw reusable-token routing results at $d=192$.  Each cell reports
mean joint route-step success (fraction) / per-inactive-key false-positive rate
(\%).  A route step succeeds iff every currently reached node is accepted and
every unreached graph node and all 512 Gaussian decoys are rejected.}
\label{tab:weighted-codebooks}
\begin{tabular}{rccc}
\toprule
Codebook seed & Uniform cumulative & Recency-heavy & Random mild \\
\midrule
4903 & $0.896/0.032$ & $0.576/0.572$ & $0.533/1.318$ \\
5004 & $0.904/0.028$ & $0.594/0.585$ & $0.541/1.282$ \\
5105 & $0.907/0.028$ & $0.593/0.557$ & $0.536/1.312$ \\
5206 & $0.913/0.026$ & $0.587/0.564$ & $0.541/1.336$ \\
5307 & $0.909/0.029$ & $0.578/0.570$ & $0.534/1.280$ \\
\midrule
Mean & $0.906/0.029$ & $0.586/0.570$ & $0.537/1.306$ \\
\bottomrule
\end{tabular}
\end{table}

The corresponding mean minimum coefficient, normalized by the uniform value,
is $1.000$, $0.814$, and $0.775$.  The monotone association between this
weakest coefficient, joint success, and false positives is the token-level
pattern predicted by the $a_{\min}$ mechanism in
Theorem~\ref{thm:minimax}.

\subsection{Software and computing resources}

We use Python 3.10, PyTorch 2.5.1, CUDA 12.4, NumPy 2.1.3, and Transformers
4.46.2 on NVIDIA A100 GPUs.  The lightweight sweep contains 55 trained models
and records 37.6 summed GPU-hours; the GPT-2 replication contains 35
single-GPU training runs.  The four-width weighted run records 0.17 GPU-hours
and its token analysis reuses the $d=192$ checkpoint.

\end{document}